\documentclass[letterpaper, 10 pt, conference]{ieeeconf} 

\IEEEoverridecommandlockouts

\title{Optimal sensor deception in stochastic environments with partial observability to mislead a robot to a decoy goal}

 \author{Hazhar Rahmani \qquad Mukulika Ghosh \qquad Syed Md Hasnayeen\thanks{
The authors are with the Dept. of Computer Science, Missouri State University, Springfield, MO, USA.  {\tt\small {\{hrahmani, mghosh, sh3739s\}}@missouristate.edu}}}

\usepackage{amsmath}
\usepackage{mathtools}
\usepackage{textcomp}
\usepackage{algorithmic}
\usepackage[ruled,vlined,noend,linesnumbered]{algorithm2e}
\usepackage[english]{babel}
\usepackage{amssymb}
\usepackage{graphicx}
\usepackage{xcolor}
\usepackage{enumerate}
\usepackage{framed}
\usepackage{cite}
\usepackage{tabularx}
\usepackage[framemethod=default,innerleftmargin=5pt,innerrightmargin=5pt,skipabove=5pt,topline=false,rightline=false,bottomline=false,linewidth=2.5pt,roundcorner=5pt]{mdframed}

\usepackage{tikz}
\usepackage{amsmath}
\usepackage{subcaption}
\usepackage{graphicx}

\usepackage{pgfplots}
\usepackage{filecontents}
\usepackage[normalem]{ulem}
\usepackage{setspace}

\usepackage{svg}

\usepackage{setspace}

\usepackage{hyperref}

\usepackage{complexity}
\let\oldNP\NP
\renewcommand{\NP}{\oldNP\xspace}

\usepackage{tikz}
\usetikzlibrary{arrows.meta, positioning}

\usetikzlibrary{automata, positioning}

\usepackage{tikz}
\usetikzlibrary{shapes.geometric, positioning, calc}
\newcommand{\pagebudget}[1]{}
\newcommand{\showtotalpagebudget}[1]{}

\newtheorem{definitionenv}{\bf Definition}
\newtheorem{lemmaenv}{Lemma}
\newtheorem{theoremenv}{Theorem}
\newtheorem{corollaryenv}{Corollary}
\newtheorem{exampleenv}{Example}
\newtheorem{assumptionenv}{Assumption}

\newenvironment{definition}{\vspace{0.05in}\begin{definitionenv}\em}{\end{definitionenv}\vspace{0.05in}}
\newenvironment{lemma}{\vspace{0.05in}\begin{lemmaenv}\em}{\end{lemmaenv}\vspace{0.05in}}
\newenvironment{theorem}{\vspace{0.05in}\begin{theoremenv}\em}{\end{theoremenv}\vspace{0.05in}}
\newenvironment{corollary}{\vspace{0.05in}\begin{corollaryenv}\em}{\end{corollaryenv}\vspace{0.05in}}

\newcommand*{\probleminternal}[4]{
	\par
	\medskip
	\noindent\fbox{\parbox{0.98\columnwidth}{
			\textbf{#4: #1} \\[0.05in]
			\renewcommand{\tabcolsep}{2pt}
			\begin{tabularx}{\linewidth}{rX}
				\emph{Input:} & #2 \\
				\emph{Output:} & #3
			\end{tabularx}
		}}
		\par
		\medskip
		\par
	}

	\newcommand*{\problem}[3]{\probleminternal{#1}{#2}{#3}{Problem}}
	\newcommand*{\decproblem}[3]{\probleminternal{#1}{#2}{#3}{Decision Problem}}

	\makeatletter
	\newcommand*{\Relbarfill@}{\arrowfill@\Relbar\Relbar\Relbar}
	\newcommand*{\xeq}[2][]{\ext@arrow 0055\Relbarfill@{#1}{#2}}
	\makeatother

	\providecolor{added}{rgb}{0,0,0}
 \providecolor{changed}{rgb}{0,0,0}
 \providecolor{altered}{rgb}{0,0,0}

\newcommand{\deleted}[1]{}

\newcommand{\Yes}{{\rm Yes}\xspace}
\newcommand{\No}{{\rm No}\xspace}

\newcommand{\OSAtoDGM}{{\rm OSA\_DGM}\xspace}
\newcommand{\OSAtoDGMDEC}{{\rm OSA\_DGM-DEC}\xspace}

\newcommand{\KNAPSACK}{{\rm 0/1-KNAPSACK-DEC}\xspace}

\newcommand*{\gobble}[1]{}

\begin{document}

\maketitle



\begin{abstract}
Deception is a common strategy adapted by autonomous systems in adversarial settings. Existing deception methods primarily focus on increasing opacity or misdirecting agents away from their goal or itinerary. In this work, we propose a deception problem aiming to mislead the robot towards a decoy goal through altering sensor events under a constrained budget of alteration. The environment along with the robot's interaction with it is modeled as a Partially Observable Markov Decision Process (POMDP), and the robot's action selection is governed by a Finite State Controller (FSC). Given a constrained budget for sensor event modifications, the objective is to compute a sensor alteration that maximizes the probability of the robot reaching a decoy goal. We establish the computational hardness of the problem by a reduction from the $0/1$ Knapsack problem and propose a Mixed Integer Linear Programming (MILP) formulation to compute optimal deception strategies. We show the efficacy of our MILP formulation via a sequence of experiments.
\end{abstract}

\section{Introduction}
Deception is a common adversarial strategy that can involve concealing intent from adversarial agents \cite{fu-covert-2024, goalobfuscation-2020} or misleading them through sensor manipulation \cite{Laforturne-2022, su2018}. This paper focuses on the latter, where an autonomous agent or robot is deliberately misled toward predetermined decoy goals via systematic sensor alterations. Since modifying sensor information incurs a cost, the system optimizes deception within a constrained budget. Unlike existing approaches that focus on optimizing the placement of decoy goals or honeypots \cite{fu-proactive-2023, milani-ssgattack-2020, anwar-honeypot-2022}, our approach assumes a fixed set of decoy goals and achieves deception by optimally swapping sensor observations (readings). This method has broad applications in autonomous system planning, including security applications and adversarial environments in multi-agent systems.
For example, consider a strategic defense scenario in an environment shown in Fig.~\ref{fig:example}. An intruder ground robot, attempts to infiltrate a protected area located in upper right cell, while avoiding detection. The environment contains 6 beacons, $S1$ to $S6$, producing distinct colors, which the ground robot relies on for localization and planning to move toward the protected area. The robot dynamic has stochastically, and so, when it does one of the actions North, East, South, and West, its actuators guarantee with a probability less than $1$ that they will move the robot to the intended cell. The defense system knows the strategy the robot uses for action selection. Accordingly, it alters the beacon identities by spoofing them, effectively misdirecting the robot away from the ammunition area and towards the middle cell containing a guard tower. This spoofing is not online and instead is offline and it is performed only one time, before the agent's execution. It can simply swap the colors produced by different beacons. The cost of beacon spoofing can include energy expenditure or the risk posed to authorized agents, and hence, there is a limited budget for spoofing.
In this paper, we answer the question that in this kind of scenarios, how we alter or swap the colors produced by some beacons to maximize the probability of directing the intruder toward the decoy goal.
Our problem can also help identify vulnerabilities in a sensor network, where a malicious agent manipulates sensors to mislead the system to an unsafe or hazardous state.
\begin{figure}
    \centering
    \begin{tikzpicture}[scale=0.8]
                \foreach \x in {0,...,5} {
                    \draw[gray] (\x,0) -- (\x,5);
                    \draw[gray] (0,\x) -- (5,\x);
                }
                %
                \foreach \x in {0,...,4} {
                    \node[below] at (\x+0.5,0) {\x};
                    \node[left] at (0,\x+0.5) {\x};
                }
                %
                \fill[blue!30,opacity=0.5] (0,0) rectangle (1,1); 
                \fill[green!30,opacity=0.5] (2,0) rectangle (3,2); 
                \fill[yellow!30,opacity=0.5] (0,2) rectangle (2,3); 
                \fill[orange!30,opacity=0.5] (4,0) rectangle (5,2); 
                \fill[purple!30,opacity=0.5] (0,4) rectangle (2,5); 
                \fill[pink!30,opacity=0.5] (4,3) rectangle (5,4); 
                \fill[cyan!30,opacity=0.5] (3,4) rectangle (4,5); 
                %
                \draw[red, thick] (2,2) -- (3,3);
                \draw[red, thick] (2,3) -- (3,2);
                %
                \draw[green, thick] (4.5,4.5) circle(0.3);
                %
                \draw[blue, thick] (0,0) rectangle (1,1);
                %
                \node at (0.5,0.5) {$s_0$};
                \node at (2.5,1) {$s_1$};
                \node at (1,2.5) {$s_2$};
                \node at (4.5,1.0) {$s_3$};
                \node at (1,4.5) {$s_4$};
                \node at (4.5,3.5) {$s_5$};
                \node at (3.5,4.5) {$s_6$};
                %
                \node[red] at (2.5,2.5) {\textbf{X}};
                \node[green] at (4.5,4.5) {\textbf{G}};

                \node at (1.5,0.5) {\includegraphics[width=0.6cm]{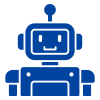}};
            \end{tikzpicture}
    \caption{An example of sensor deception. The agent at bottom left cell is tasked to navigate to the goal at the top right cell. 
    Because of photoelasticity in the robot's dynamic, the current state, the position of the robot, is not observable to the robot. The range sensors provide partial observability.
    The system will alter the sensors to mislead the agent to the decoy goal state in the middle cell containing the security.}
    \label{fig:example}
\end{figure}
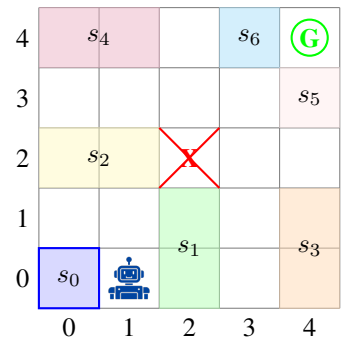

In this paper, we consider the sensor deception problem in a stochastic environment by modeling the interaction between an autonomous agent (or robot) and its environment as a Partially Observable Markov Decision Process (POMDP). The robot lacks direct knowledge of its true state and relies on sensor-generated observations. Those observations can be strategically altered by an agent at a cost. The robot selects actions based on its current observations and memory, which we model as a Finite State Controller (FSC) \cite{ahmadi2020stochastic}. The objective is to maximize the probability of misleading the robot to the predetermined decoy goal when it has limited capability of modifying the sensor readings.
We establish the computational hardness of our problem by a reduction from the 0/1 Knapsack problem. Furthermore, we formulate the optimal sensor deception problem as a Mixed Integer Linear Programming (MILP) model, offering a structured and scalable solution for synthesizing deception strategies that achieve optimal cost-efficiency. 

The contributions of this paper are summarized as follows:
\begin{itemize}
\item A novel deception problem formulation, which aims to compute a sensor alteration maximizing the probability of misdirecting a robot two a predetermined decoy goal, under a limited cost budget of sensor alteration.
\item The $\NP$-hardness of our sensor alteration problem, proved by reduction from Knapsack problem.
\item A scalable MILP-based solution of the sensor alteration problem.
\end{itemize}
%
After discussing related work in Section~\ref{sec:rel}, we provide basic definitions and present our problem formulation in Section~\ref{sec:def}. In Section~\ref{sec:pro}, we present verification algorithm. In Section~\ref{sec:har}, we present our $\NP$-hardness result, and in Section~\ref{sec:milp}, we provide our MILP formulation. Section~\ref{sec:case} reports results of our experiments, and Section~\ref{sec:con}, we conclude the paper and draw future research directions.
%
\section{Related Work}
\label{sec:rel}

%


%
Deception of adversary or opponent by altering sensor information has been achieved in form of opacity \cite{opacity, JACOB2016135} and deceptive control \cite{zhang2022selection, milani-ssgattack-2020}.  Deceptive control is implemented through adding, removal or altering sensors or observations. In this paper, we focus on deception planning. Optimal placement of decoy targets is considered in \cite{fu-proactive-2023, milani-ssgattack-2020} with linear program solution to optimize deception with \cite{fu-proactive-2023} utilizing probabilistic attack graphs compared to deterministic one in \cite{milani-ssgattack-2020}. Goal obfuscation considers hiding true goal among fake goals to implement deception \cite{goalobfuscation-2020, bernardini2020optimization}.  Savas et. al.  use Linear Programming to maximize  adversary teams' deception about its true goal while ensuring the highest probability of achieving its true goal in \cite{Savas_Verginis_Topcu_2022}. The observer’s predictions are formulated as probability distributions over true goals, guided by the maximum entropy principle. Game theory based approaches have been used in deploying decoys for deception in hypergames \cite{kulkarni-decoy-2020, kulkarni-decoylabel-2020, li-games-2023}. Unlike the approaches mentioned above, we model a sensor deception by the alterations of sensors to lead the agent to a finite set of predetermined decoy targets while maintaining the cost incurred by sensor alterations under a limited budget.

Markov Decision Process (MDP) has been widely used in deceptive planning \cite{chen2024deceptive, ufuk-2023, fu-covert-2024, Savas_Verginis_Topcu_2022}. Karabag et al. \cite{ufuk-2023} explore deception by an agent through partial observability of the supervisor, modeling the environment as an MDP. The agent optimizes the generation of a deceptive policy, quantifying deception using Kullback-Leibler (KL) divergence to ensure that the agent's actions remain indistinguishable from the supervisor's expectations while achieving its covert objective. In \cite{ufuk-2025}, the method is extended to team deception, where each agent is modeled as an MDP, and a centralized deceptive policy synthesis is applied while maintaining decentralized execution. Ma et al. \cite{fu-covert-2024} investigate covert deception planning in a stochastic environment, modeled as an MDP, where the agent aims to maximize the discounted total reward while maintaining covert behavior. They demonstrate that finite-memory policies can outperform Markovian policies in this constrained MDP framework and develop a primal-dual gradient-based method for synthesizing optimal and covert Markov policies. In our approach, we model the interaction of the deceived agent with the environment as partially observable Markov Decision process (POMDP) and the selection of action by the agent as finite state controller.

Another closely related work by Meira-Góes et al. in \cite{Lafortune-2019} investigates sensor manipulation in a stochastic environment by modeling the system as a probabilistic automaton. The authors leverage concepts from stochastic games to synthesize an attack policy that maximizes the likelihood of the system being in an unsafe state. In \cite{Laforturne-2022}, the authors extend their work by modeling the interaction between the attacker and the system as an MDP, and also  incorporating penalties into the attack model as a second problem. While these work maximizes the unsafety of the system, our proposed work focuses on sensor manipulation to maximize the likelihood of the agent landing in decoy state.

\section{Definitions and Problem Statement}
\label{sec:def}
To model the interaction of the robot with the environment we use a discrete structure defined as follows.
\begin{definition}
A \emph{partially observable Markov decision process (POMDP)} is a tuple $\mathcal{M}=(S, A, \mathbf{P}, s_0, \Omega, O)$ in which
\begin{itemize}
    \item $S$ is a finite set of states,
    \item $A$ is a finite set of actions,
    \item $\mathbf{P}: S \times A \times S \rightarrow [0, 1]$ is a transition probability function such that for each $s, s' \in S$, $a \in A$, $\mathbf{P}(s, a, s')$ is 
    the probability that the system transitions to state $s'$ by doing action $a$ at state $s$, and it holds that $\Sigma_{s' \in S}\mathbf{P}(s, a, s') = 1$ for all $s \in S$ and $a \in A$,
    \item $s_0$ is the initial state,
    \item $\Omega$ is a finite set of observations,
    \item $O: S \rightarrow \Omega$ is the observation function such that for each state $s \in S$, $O(s)$ is the observation the system emits when it enters $s$.
\end{itemize}
\label{def:pomdp}
\end{definition}
%

A POMDP defines how the state of the system changes in response to the robot's actions.
At each step, the robot receives an observation produced by the system based on the current state. Since the same observation could be produced by multiple states, the robot cannot directly observe the current state.
The robot's action selection process is governed by a finite-state controller (FSC), formally defined as follows.

\begin{definition}
    A \emph{finite-state controller} for a POMDP $\mathcal{M}=(S, A, \mathbf{P}, s_0, \Omega, O)$ is a tuple $\mathcal{C} = (N, n_0, \gamma, \delta)$ in which
    \begin{itemize}
        \item $N$ is a finite set of memory nodes,
        \item $n_0$ is the initial node, 
        \item $\gamma: N \times \Omega \rightarrow A$ is the \emph{action selection} function where for each node $n \in N$ and observation $o \in \Omega$, it chooses an action $\gamma(n , o)$ when the controller is in state $n$ and the robot has perceived an observation $o$,
        \item $\delta: N \times \Omega \rightarrow N$ is the memory update function, telling the controller to transition to node $\delta(n, o)$ by receiving observation $o$ at node $n$.
    \end{itemize}
    \label{def:fsc}
\end{definition}

The robot uses the finite-state controller $\mathcal{C} = (N, n_0, \gamma, \delta)$ to choose actions at every time-step. 
At the initial time step $0$, the POMDP is in state $s_0$ and the controller has memory $n_0$. 
Because of partial observability, the robot generally does not explicitly know the current state of the POMDP. However, based on the perceived observation associated with the current state, it can reduce uncertainty about the state of the POMDP.
At each time-step $k \geq 0$, the controller generates action $a_k = \gamma(n_k, O(s_k))$, which must be executed by the robot. Upon executing this action, the POMDP transitions from state $s_k$ to $s_{k+1}$ at time step $k+1$. 
When the POMDP enters $s_{k+1}$, an observation $O(s_{k+1})$ is emitted. 
State $s_{k+1}$ is produced randomly from every state $s'$ for which $\mathbf{P}(s_k, a_k, s') > 0$. 
%


In this paper, we study a problem in which the sensors have been attacked by an adversarial agent, causing them to produce incorrect observations.
The purpose of the sensor attack is to mislead the robot into a decoy goal. 
We consider only \emph{alteration} attacks.
 \begin{definition}[Observation alteration]
\label{def:obs_alt}
An observation alteration is a function $\alpha: \Omega \rightarrow \Omega $ where for each
$o \in \Omega$, the robot receives observation $\alpha(o)$ instead of $o$ whenever it is supposed to receive $o$.
\end{definition}
Consider that the sensor attack is in a sense an offline attack rather than an online attack, meaning that it is performed before any system execution. Each attack incurs a cost, defined by a function: Given the observation alteration cost function $c: \Omega \times \Omega  \rightarrow \mathbb{R}_{\geq 0}$, where for each $o_1, o_2 \in \Omega$, $c(o_1, o_2)$ is the cost of altering observation $o_1$ to observation $o_2$. 

Let $\mathbf{A}$ be the set of all sensor alterations defined over $\Omega$. 
The cost associated with a sensor alteration is identified using function $C: \mathbf{A} \rightarrow \mathbb{R}_{\geq 0}$, assigning cost $C(\alpha) = \sum_{o \in \Omega} c(o, \alpha(o))$ to sensor alteration $\alpha$.

%

%

The adversary's objective is to make an observation alteration that maximizes the probability of misleading the agent to the decoy while ensuring the cost of the observation alteration is no greater than a cost budget.
%
%
%
%
%
\problem{Optimal Sensor Alteration for Decoy Goal Misleading (OSA\_DGM)}
{A POMDP $\mathcal{M}=(S, \alpha, \mathbf{P}, s_0, \Omega, O)$,  a decoy goal $S_D \subseteq S$, a reference controller $\mathcal{C}=(N, n_0, \gamma, \delta)$, a cost alteration function $c$, and a cost budget $B \in \mathbb{R}_{\geq 0}$.
}
{A sensor alteration \( \alpha \) with cost at most \( B \) for which, $\textbf{Pr}_{\text{reach}}(S_D, \alpha)$ is maximum.
}

\section{Misleading probability of a sensor alteration}
\label{sec:pro}
In this section, we introduce a product automaton construction, which can be used for a hardness result presented in the next section, and for computing the probability of misleading the robot to a decoy goal on a sensor alteration.
%
\begin{definition}
\label{def:prod_aut}
    The product of POMDP $\mathcal{M}=(S, A, \mathbf{P}, s_0, \Omega, O)$ and finite-state controller $\mathcal{C} = (N, n_0, \gamma, \delta)$ under sensor alteration $\alpha: \Omega \rightarrow \Omega$ for decoy goal $S_D \subseteq S$ is a tuple $\mathcal{P} = (Q, q_0, \mathbf{T}, Q_G)$ in which
    \begin{itemize}
        \item $Q = S \times N$ is the state space,
        \item $q_0 = (s_0, n_0)$ is the initial state,
        \item $Q_D = S_D \times N$ is the set of goal states, and
        \item $\mathbf{T}: Q \times Q \rightarrow [0, 1]$ is the transition function such that for each states $(s, n), (s', n') \in Q$,
        \begin{align}
   \mathbf{T}((s, n), (s', n')) =\begin{cases}
			\mathbf{P}(s, s', a) & \text{if $\delta(n, \alpha(O(s))) = n'$ }\\
             & \text{and }\\
              & \text{$\gamma(n, \alpha(O(s))) = a$}\\
            0 & \text{otherwise}.
		 \end{cases}
        \end{align}
    \end{itemize}
\end{definition}
Note that this automaton is, in fact, a \emph{goal Markov chain}---a Markov chain with the set of goal states $Q_D$.
Each state of it is a tuple $(s, n)$ in which $s$ is the current state of the world and $n$ is current memory node of the controller. 
The system produces observation $O(s)$ and the attacker alters this sensor reading to $\alpha(O(s))$. The robot receives the modified observation $\alpha(O(s))$ and uses it to select action $a= \gamma(n , \alpha(O(s)))$, generated by the controller. The controller then transitions to node $n' = \gamma(n , \alpha(O(s)))$, and the POMDP transitions from state $s$ to a state $s'$ stochastically, based on $\mathbf{P}(s, n, .)$.

To compute the probability of reaching the decoy, one can introduce a variable $z_q$ for each $q \in Q$, and set $z_q = 1$ if $q \in Q_D$, and otherwise,
\[
 z_q = \sum_{q' \in Q} \mathbf{T}(q, q')z_{q'}. 
\]
By solving this Bellman equation using standard methods, such as the method in Chapter 10 of \cite{baier2008principles}, the probability of misleading the robot to the decoy goal is given by $z_{q_0}$.
%
%
%
%
%
%
%
\section{Hardness Results}
\label{sec:har}
In this section, we present our hardness result.
%
%

First, we consider the decision variant of our problem.

\decproblem{Optimal Sensor Alteration for Decoy Goal Misleading (\OSAtoDGMDEC)}
{A POMDP $\mathcal{M}=(S, A, \mathbf{P}, s_0, \Omega, O)$, a decoy goal $S_D \subseteq S$, a reference controller $\mathcal{C}=(N, n_0, \gamma, \delta)$, a cost alteration function $c$, a cost budget $B \in \mathbb{R}_{\geq 0}$, and a real number $r \in \mathbb{R}_{\geq 0}$.
}
{\Yes, if there exists a sensor alteration \( \alpha \) such that \( C(\alpha) \leq B \) and $\textbf{Pr}_{\text{reach}}(S_D, \alpha) \geq r$; and \No, otherwise.}
%
%
%
%
%
%
%
\begin{lemma}
\label{lem:np}
    \OSAtoDGMDEC $\in \NP$.
\end{lemma}
\begin{proof}
    Let $x : \langle \mathcal{M} :=(S, A, \mathbf{P}, s_0, \Omega, O), \mathcal{C}:=(N, n_0, \gamma, \delta), S_D, B, r  \rangle$ be an instance of \OSAtoDGMDEC. We assume the sensor alteration $\alpha$ is given as the certificate. We need to prove that in polynomial time to the size of $x$ we can verify if $C(\alpha) \leq B$ and
    $\textbf{Pr}_{\text{reach}}(S_D, \alpha) \geq r$. 

    We construct $\mathcal{P}$ using the construction in Definition~\ref{def:prod_aut}. The running time of this construction is $\mathcal{O}(|S||N||A||\Omega|+|S|^2|N|^2)$, which is polynomial to the size $x$.
    Then, it takes a time polynomial to the size of $\mathcal{P}$, to compute the probability of reaching $Q_G$ in $\mathcal{P}$, which represented $\textbf{Pr}_{\text{reach}}(S_D, \alpha)$.
    %
    %
    Trivially, checking whether $\textbf{Pr}_{\text{reach}}(S_D, \alpha) \geq r$ and $C(\alpha) \leq B$ takes a polynomial time, and this completes the proof.
\end{proof}

%

Next, we consider a well known problem.
\decproblem{0/1 Knapsack Problem (\KNAPSACK)}
{$n$ items with weights $W = [w_1, w_2, \cdots, w_n]$ and values $V=[v_1, v_2, \cdots, v_n]$, a knapsack with capacity $P$, a positive real number $L \geq 0$.}
{\Yes if there is a set $I \subseteq \{1, 2, \cdots, n\}$ such that $\sum_{i \in I} w_i \leq P$ and $\sum_{i \in I} v_i \geq L$, and \No otherwise.}
In words, this problem asks whether the knapsack can be filled, either fully or partially, with a subset of items whose total weight is at least \( L \).

Next, we show that our problem is computationally hard.

\begin{theorem}
\label{thr:nphard}
    \OSAtoDGMDEC $\in \NP$-hard.
\end{theorem}
\begin{proof}
    By reduction from the 0/1 Knapsack problem.
    %
        
   Given an instance 
    \[
    x = \langle W := [w_1, \cdots, w_n], V := [v_1, \cdots, v_n], P, L \rangle
    \]
     of the \KNAPSACK problem, we construct an instance
     \[
     y = \langle \mathcal{M}, S_D := \{s_{\bot}\}, \mathcal{C}, c, B, r \rangle
     \]
     of \OSAtoDGMDEC in which for $\mathcal{M} = (S, A, \mathbf{P}, s_0, \Omega, O)$, 
     \begin{itemize}
         \item $S = \{s_0\} \cup \{s_i \mid i \in \{1, \cdots, n\} \} \cup \{s_{\clubsuit}\} \cup \{s_{\top} \} \cup \{s_{\bot}\}$ 
         \item $A = \{a, b\}$ 
         \item $\Omega = \{o_0\} \cup \{o_i \mid i \in \{1, \cdots, n\}\} \cup \{o_{\clubsuit}\} \cup \{o_{\top}\} \cup \{o_{\bot}\}$ 
         %
         \item for each $j \in \{0\} \cup \{1, \cdots, n\} \cup \{\bot, \top, \clubsuit\}$,
         \[
           O(s_j) = o_{j},
         \]
         \item for each $s, s' \in S$ and $t \in A$,
         \begin{multline*}
             \mathbf{P}(s, t, s') = \\
        \begin{cases}
            \frac{v_i}{2\sum_{j=1}^n v_j} & \text{if } (s = s_0, t = a, s' = s_i) \text{ for } 1 \leq i \leq n \\
            \frac{1}{2} & \text{if } (s = s_0, t = a, s' = s_{\clubsuit}) \\
            1 & \text{if } (s = s_0, t = b, s' = s_0) \\
            1 & \text{if } (s = s_i, t = a, s' = s_{\top}) \text{ for } 1 \leq i \leq n \\
            1 & \text{if } (s = s_i, t = b, s' = s_{\bot}) \text{ for } 1 \leq i \leq n \\
            1 & \text{if } (s = s_{\clubsuit}, t = b, s' = s_{\top}) \\
            1 & \text{if } (s = s_{\clubsuit}, t = a, s' = s_{\bot}) \\
            1 & \text{if } (s = s_{\top}, s' = s_{\top}, t \in A)   \\
            1 & \text{if } (s = s_{\bot}, s' = s_{\bot}, t \in A)   \\
            0 & \text{otherwise}
        \end{cases}
         \end{multline*}
     \end{itemize}
    and for $\mathcal{C} = (N, n_0, \gamma, \delta)$, we have 
    \begin{itemize}
        \item $N = \{n_0, n_1, n_2 \}$,
        \item For each $n \in N$ and $o \in \Omega$,
        \begin{align*}
            \gamma(n, o) =
        \begin{cases}
            a & \text{if } n = n_0, o = o_0    \\
            a & \text{if } n = n_1, o = o_i \text{ for } 1 \leq i \leq n \\
            b & \text{otherwise}.
        \end{cases}
        \end{align*}
        \item $\delta$ is defined as follows: For each $n \in N$ and $o \in \Omega$,
        \begin{itemize}
            \item $\delta(n, o) = n_1$ if $n_0$ and $o = o_0$,
            \item and otherwise, $\delta(n, o) = n$
        \end{itemize}
        \item for each $n \in N$ and $o \in \Omega$,
        \begin{align*}
            \delta(n, o) =
        \begin{cases}
            n_1 & \text{if } n = n_0, o = o_0, \\
            n_2 & \text{otherwise}
        \end{cases}
        \end{align*}
    $S_D = \{s_{\bot}\}$, 
    \end{itemize} \
    we define $c$ such that for each $o, o' \in \Omega$,
    \begin{align*}
        c(o, o') =
        \begin{cases}
            0 & \text{if } o = o' \\
            w_i & \text{if } o = o_i \text{ for } 1 \leq i \leq n, o' = o_{\clubsuit} \\
            \infty & \text{otherwise}
        \end{cases}
    \end{align*}
    and we set
    $B = P$ and $r = \frac{L}{2 \sum_{i=1}^n v_i}$.

    Clearly, this reduction takes a polynomial time, and hence, we only need to show that the reduction is correct, i.e., \KNAPSACK produces \Yes for $x$  iff  \OSAtoDGMDEC produces \Yes for $y$.
    
    ($\Rightarrow$) Assume that \KNAPSACK produces \Yes for $x$, that is, there exists set $I \subseteq \{1, \cdots, n\}$ such that  $\sum_{i \in I} w_i \leq P$ and $\sum_{i \in I} v_i \geq L$.
    We construct the sensor alteration $\alpha$ such that for each observation $o \in \Omega$, 
    %
    \begin{align*}
    \alpha(o)=
        \begin{cases}
            a_{\clubsuit} & \text{if } o=o_i \text{ for an integer } i \in I \\
            o & \text{otherwise}
        \end{cases}
    \end{align*}
    This means when the system enters any state $s_i$ for $i \in I$, the robot is deceived to think it is in $s_{\clubsuit}$, causing the controller to give action $b$ instead of $a$ to the robot to execute, and by executing that action, the system enters $s_{\bot}$.
    Because $s_{\bot}$ is reached only by doing action $b$ from the $s_j$'s for $j \in \{1, \cdots, n\}$, the probability of reaching the decoy goal $\{s_{\bot}\}$ under sensor alteration $\alpha$ is computed
    \[
   \textbf{Pr}_{\text{reach}}(S_D, \alpha)=\sum_{i \in I} \frac{v_i}{2\sum_{j=1}^nv_j} = \frac{\sum_{i \in I} v_i}{2\sum_{j=1}^nv_j}.
    \]
    , and because $\sum_{i \in I} v_i \geq L$,
    \begin{align} \label{eq:pr_reach_gt_r}
        \textbf{Pr}_{\text{reach}}(S_D, \alpha) \geq \frac{L}{2\sum_{j=1}^nv_j} = r.
    \end{align}
    Also,
    \begin{align} \label{eq:C_alphba}
        C(\alpha) &= \sum_{o \in \Omega} c(o, \alpha(o)) \notag \\
           &= \sum_{o \in \Omega \setminus \{s_i \mid i \in I\} } c(o, \alpha(o))+\sum_{i \in I} c(o_i, \alpha(o)) \notag \\
           &= \sum_{o \in \Omega \setminus \{s_i \mid i \in I\} } c(o, o)+\sum_{i \in I} c(o_i, o_{\clubsuit}) \notag \\
           &= \sum_{o \in \Omega \setminus \{s_i \mid i \in I\} } 0+\sum_{i \in I} w_i \notag \\
           &=\sum_{i \in I} w_i,
    \end{align}
    and because $\sum_{i \in I} w_i \leq P$ and that $B=P$, it holds that $C(\alpha) \leq B$.
    This combined with (\ref{eq:pr_reach_gt_r}) implies that \OSAtoDGMDEC produces \Yes for $y$.  

    ($\Leftarrow$) Assume \OSAtoDGMDEC produces \Yes for $y$, meaning that, there is a sensor alteration $\alpha$ for which $\textbf{Pr}_{\text{reach}}(S_D, \alpha) \geq r$ and that $C(\alpha) \leq B$.
    State $s_{\bot}$ can be reached either by doing action $b$ at the $s_i$'s for $i \in \{1, \cdots, n\}$ or by doing $a$ at $s_{\clubsuit}$. 
    The cost of altering $o_{\clubsuit}$ to any other observation is $\infty$. Therefore, if the system enters $s_{\clubsuit}$, it will certainly perform action $b$, causing it to transition to $s_{\top}$ rather than $s_{\bot}$ in the next time step.
    Thus, $s_{\bot}$ cannot be reached from $s_{\clubsuit}$.
    The system enters $s_{\bot}$ from a state $s_i$ for $i \in \{1, \cdots, n\}$ only when $o_i$ is mapped to $o_{\clubsuit}$ by $\alpha$.
    Because $\textbf{Pr}_{\text{reach}}(S_D, \alpha) \geq = r$, there must be a set $I \subseteq $ such that among all the $o_i$'s, only those for which $i \in I$, it has been set $\alpha(o_i) = o_{\clubsuit}$.
    By assumption $C(\alpha) \leq B$, and since $\sum_{i \in I} w_i$ by (\ref{eq:C_alphba}), and that $B = P$, it holds that $\sum_{i \in I} w_i \leq P$.
    Also, because
    \begin{align}
        \mathbf{Pr}_{\text{reach}}(S_D, \alpha) = \sum_{i \in I} \frac{v_i}{\sum_{j=1}^n w_j} \geq r = \frac{L}{\sum_{j=1}^n w_j},
    \end{align}
    it holds that $\sum_{i \in I} v_i \geq L$.
    This combined with that $\sum_{i \in I} w_i \leq P$ proves that $I$ yields \Yes for instance $x$ of \KNAPSACK.
\begin{figure}[ht!]
    \centering

    \begin{subfigure}{0.35\textwidth}
        \centering
        \begin{tikzpicture}[scale=0.6]

            
            \fill[blue!50] (0,0) rectangle (1,1);
            \draw[thick] (0,0) rectangle (1,1);
            \node at (0.5,0.5) {\textbf{1}};
            \node at (0.5,-0.5) {\small 20};

            \fill[red!50] (1.5,0) rectangle (2.5,2);
            \draw[thick] (1.5,0) rectangle (2.5,2);
            \node at (2,1) {\textbf{2}};
            \node at (2,-0.5) {\small 30};

            \fill[green!50] (3,0) rectangle (4,3);
            \draw[thick] (3,0) rectangle (4,3);
            \node at (3.5,1.5) {\textbf{3}};
            \node at (3.5,-0.5) {\small 40};

            \fill[orange!50] (4.5,0) rectangle (5.5,4);
            \draw[thick] (4.5,0) rectangle (5.5,4);
            \node at (5,2) {\textbf{4}};
            \node at (5,-0.5) {\small 50};

            \fill[purple!50] (6,0) rectangle (7,5);
            \draw[thick] (6,0) rectangle (7,5);
            \node at (6.5,2.5) {\textbf{5}};
            \node at (6.5,-0.5) {\small 60};

            \draw[thick] (8.5,0) rectangle (9.5,7);
            \node[rotate=90] at (9,3.5) {Knapsack with capacity 7};

        \end{tikzpicture}
        \caption{}
    \end{subfigure}
    \hfill
    \begin{subfigure}{0.1\textwidth}
        \centering
        \begin{tikzpicture}[scale=0.6]

            \draw[thick] (0,0) rectangle (1,7) node[pos=.5] {};

            \fill[blue!50] (0,0) rectangle (1,1);
            \draw[thick] (0,0) rectangle (1,1);
            \node at (0.5,0.5) {\textbf{1}};
            \node at (-0.7,0.5) {\small 20};

            \fill[red!50] (0,1) rectangle (1,3);
            \draw[thick] (0,1) rectangle (1,3);
            \node at (0.5,2) {\textbf{2}};
            \node at (-0.7,2) {\small 30};

            \fill[orange!50] (0,3) rectangle (1,7);
            \draw[thick] (0,3) rectangle (1,7);
            \node at (0.5,5) {\textbf{4}};
            \node at (-0.7,5) {\small 50};

        \end{tikzpicture}
        \caption{}
    \end{subfigure}
    \begin{subfigure}{0.35\textwidth}
        \centering
        \begin{center}
    \begin{tikzpicture}[shorten >=1pt, node distance=1.8cm, on grid, auto, initial text=, every state/.style={draw, minimum size=0.5cm}]
        \node[state] (s0) {$s_0$};
        \node[state] (s3) [below=of s0, shift={(0,0cm)}] {$s_3$};
        \node[state] (s2) [left=of s3, shift={(0.7cm,0)}] {$s_2$};
        \node[state] (s1) [left=of s2, shift={(0.7cm,0)}] {$s_1$};
        \node[state] (s4) [right=of s3, shift={(-0.7cm,0)}] {$s_4$};
        \node[state] (s5) [right=of s4, shift={(-0.7cm,0)}] {$s_5$};
        \node[state] (sc) [right=of s5, shift={(-0.7cm,0)}] {$s_{\clubsuit}$};

        \node[state] (st) [below right=of s2, shift={(-0.6cm,0)}] {$s_{\top}$};
        \node[state] (sb) [right=of st] {$s_{\bot}$};

        \coordinate (start) at (0,0.7);

        \path[->]
        (start) edge node {} (s0.north)
        (s0) edge [loop right, dashed] node {} (s0)
        
        (s0) edge node [left] {$\frac{20}{400}$} (s1)
        (s0) edge node [left] {$\frac{30}{400}$} (s2)
        (s0) edge node [left] {$\frac{40}{400}$} (s3)
        (s0) edge node [left] {$\frac{50}{400}$} (s4)
        (s0) edge node [left] {$\frac{60}{400}$} (s5)
        (s0) edge node [left] {$\frac{200}{400}$} (sc)

        (s1) edge node [left] {} (st)
             edge [dashed] node [left] {} (sb)
        (s2) edge node [left] {} (st)
             edge [dashed] node [left] {} (sb)
        (s3) edge node [left] {} (st)
             edge [dashed] node [left] {} (sb)
        (s4) edge node [left] {} (st)
             edge [dashed] node [left] {} (sb)
        (s5) edge node [left] {} (st)
             edge [dashed] node [left] {} (sb)
        (sc) edge [bend left] node [left] {} (sb)
             edge [dashed, bend left=50] node [left] {} (st);
\end{tikzpicture}
\end{center}
\caption{}
    \end{subfigure}
        \begin{subfigure}{0.1\textwidth}
        \centering
        \begin{center}
    \begin{tikzpicture}[shorten >=1pt, node distance=2cm, on grid, auto, initial text=, every state/.style={draw, minimum size=0.5cm}]
        
        \node[state] (n0) {$n_0$};
        \node[state] (n1) [below=of n0, yshift=0.5cm] {$n_1$};
        \node[state] (n2) [below=of n1] {$n_2$};
        
        \path[->]

    (0,0.7) edge node {} (n0.north)
        (n0) edge [] node {$o_0: a$} (n1)
        (n1) edge [] node [align=center] {
            $o_1: a$ \\ 
            $o_2: a$ \\
            $o_3: a$ 
            } 
        (n2)
             edge [left] node [align=center] {
              $o_4: a$ \\
              $o_5: a$ \\
              $o_{\clubsuit}: b$  
            } (n2);
    \end{tikzpicture}
        

\end{center}
\caption{}
    \end{subfigure}

    \caption{ (a) An instance of the 0/1 knapsack problem. There are $5$ items with weights $W = [1, 2, 3, 4, 5]$ and values $V = [20, 30, 40, 50, 60]$. The capacity of the knapsack is $7$ (b) Optimal solution to the instance of the 0/1 knapsack problem. The knapsack's total weight is 7, and the total value is 100. (c) The POMDP of the instance of our problem, the \OSAtoDGMDEC problem, constructed by our reduction for the instance of the \KNAPSACK in Part (a) of this figure. The solid edges are transitions that take place with action $a$ and the dashed arrows are transitions for action $b$. All the transitions missing probability labels, use probability $1$. We omitted those labels to reduce visual clutter. States $s_{\bot}$ and $s_{\top}$ are absorbing states. Their outgoing transitions are omitted to reduce visual clutter. (d) The finite-state controller of the instance of our problem constructed by our reduction for the instance of the \KNAPSACK in Part (a) of this figure. All the missing transitions enters $n_2$ and choose action $b$. 
    }
    \label{fig:reduction}
\end{figure}
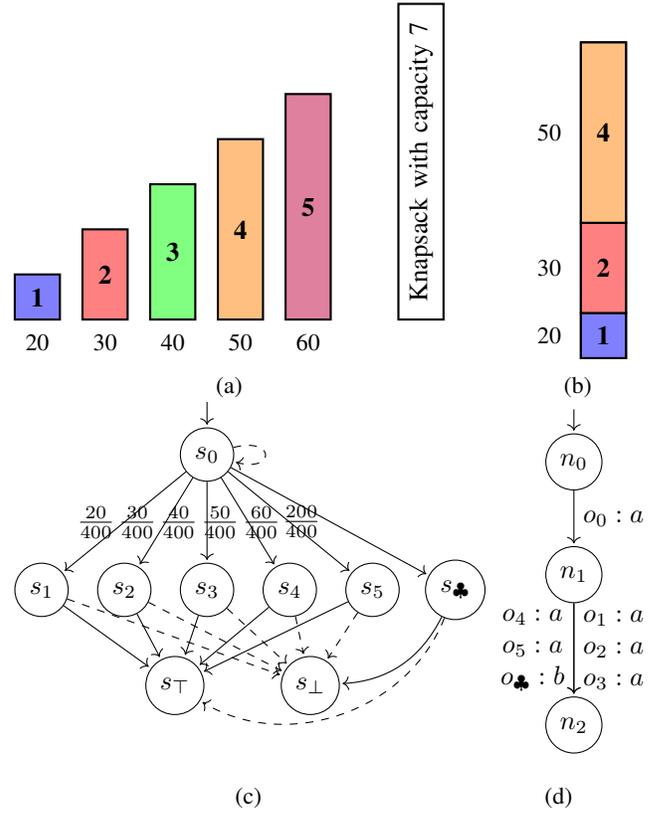

\end{proof}

Therefore, the following result is implied from Lemma~\ref{lem:np} and Theorem~\ref{thr:nphard}.
\begin{theorem}
    \OSAtoDGMDEC $\in \NP$-complete.
\end{theorem}
\begin{proof}
    Combine Lemma \ref{lem:np} and Theorem~\ref{thr:nphard}.
\end{proof}

\begin{corollary}
    \OSAtoDGM $\in \NP$-hard.
\end{corollary}

Hence, under the assumption that $\P \neq \NP$, we cannot find in polynomial time, a sensor alteration with a given cost budget that maximizes the probability of misleading the robot to the decoy goal.
%
%
%
\section{\OSAtoDGM via Mixed Integer Linear Programming}
\label{sec:milp}
In this section, we provide a mixed integer linear programming approach to solve the optimal sensor alteration problem.

The idea of our programming formulation is to have binary variables by assigning values to which, a sensor alteration is synthesized, and then to have certain variables each representing the probability of reaching to the decoy goal by the sensor alteration from a state of a product automaton. Each state of this product automaton is a tuple $(s, n, o)$ representing a situation when the POMDP is in state $s$, current memory node of the controller is $n$, and observation $O(s)$ is altered to $o$ by the sensor alteration.
%


More precisely, we introduce the following variables:
\begin{itemize}
\item  A binary variable $x_{o, o'}$ for each $o, o' \in \Omega$: Variable $x_{o, o'}$ will receive $1$ iff observation $o$ is altered to observation $o'$ by the sensor alteration.
    \item A continuous variables $z_{s, n, o} \in [0, 1]$ for each  state $s \in S$, node $n \in N$, and observation $o \in \Omega$: The value of $z_{s, n, o}$ will be the probability of misleading the robot to the decoy goal when the system is in state $s$, the controller is in node $n$, and the observation emitted when the system is in $s$, $O(s)$, is altered to $o$ under the sensor alteration.
\end{itemize}

Using these variables, we make the following programming model.
\begin{spacing}{0.5}
    \begin{mdframed}
        \noindent Maximize:\vspace{-0.5em}
        \begin{equation}\label{ilp:obj_2}
            z_{s_0, n_0, O(s_0)}
            \vspace{-0.5em}
        \end{equation}
        Subject to:
        \begin{itemize}
          \item 
          \begin{equation} \label{ilp:cost_budget}
              \sum_{o, o' \in \Omega} x_{o, o'} \cdot c(o, o') \leq B
          \end{equation}
            \item 
            \begin{equation}\label{ilp:initialstate_2}
                    x_{O(s_0), O(s_0)} = 1 
                \end{equation}
            


            \item For each $s \in S_D$, $n \in N$, and $o \in \Omega$, 
            \begin{align}\label{ilp:decoy_2}
                  z_{s, n, o} = x_{O(s), o}
            \end{align}
            
            \item For each $s \in S \setminus S_D$, $n \in N$, and $o \in \Omega$
            \begin{align}\label{ilp:bellman_2}
    z_{s, n, o} = \sum_{\substack{s' \in S, \\ n' \in N, \\ o' \in \Omega}} 
     x_{O(s), o}  \mathbf{T}((s, n, o), (s', n', o')) 
     z_{s', n', o'}
\end{align}

\item For each $o, o' \in \Omega$, 
            \begin{equation}\label{ilp:x_binary}
               x_{o, o'} \in \{0, 1\}
        \end{equation}

        \item For each $s \in S$, $n \in N$, and $o \in \Omega$, 
            \begin{equation}\label{ilp:x_binary}
               z_{s, n, o} \in [0, 1]
        \end{equation}
            
        \end{itemize}
    \end{mdframed}
    \end{spacing}
    \medskip

    The objective (\ref{ilp:obj_2}) is to maximize the probability of misleading the robot to the decoy goal $S_D$.
    Constraint (\ref{ilp:cost_budget}) ensures that the cost of the alteration is no greater than the cost budget $B$.
    Constraint (\ref{ilp:initialstate_2}) asserts that the observation associated with the initial state of the POMDP cannot be altered. This is because the robot know that any execution of the POMDP starts from $s_0$.
    Constraints (\ref{ilp:decoy_2}) sets the base case for the Bellman equation by setting to be $1$, the probability of reaching the decoy for a state $(s, n, o)$ where $s$ is a decoy state and observation $O(s)$ is altered to $o$ by the sensor alteration.
    Constraints (\ref{ilp:bellman_2}) simply implement the Bellman equation.

    This programming would be an MILP if Constraints (\ref{ilp:bellman_2}) were linear. Note that although $\mathbf{T}((s, n, o), (s', n', o'))$, because both $x_{O(s), o}$ and $z_{s', n', o'}$ variables, the product of these three is not a linear term. Thus, we need to linear these constraints.
    
    To linearize them, we first introduce a continuous variable $l_{s, o, s', n', o'}$ for each $s, s' \in S$, $n, n' \in N$, and $o, o' \in \Omega$. 
    The range of this variable is $[0, 1]$.
    Then, we introduce the following additional constraints to make $l_{s, o, s', n', o'}$ receive value $x_{O(s), o} \cdot z_{s', n', o'}$.
    \begin{spacing}{0.5}
    \begin{mdframed}
        \begin{itemize}
          \item For each $s, s' \in S$, $n' \in N$, and $o, o' \in \Omega$,
            \begin{equation}\label{ilp:linear_1}
                    l_{s, o, s', n', o'} \geq 0 
                \end{equation}

            \begin{equation}\label{ilp:linear_2}
                    l_{s, o, s', n', o'} \leq z_{s', n', o'} 
                \end{equation}

            \begin{equation}\label{ilp:linear_2}
                    l_{s, o, s', n', o'} \leq x_{O(s), o} 
                \end{equation}

            \begin{align}\label{ilp:linear_2}
                    l_{s, o, s', n', o'} \geq & z_{s', n', o'} - (1- x_{O(s), o})
                \end{align}
        \end{itemize}
    \end{mdframed}
    \end{spacing}
    \medskip
    Finally, we replace Constraints~\ref{ilp:bellman_2} with the following constraints:
    \begin{spacing}{0.5}
    \begin{mdframed}
        \begin{itemize}
            \item For each $s \in S \setminus S_D$, $n \in N$, and $o \in \Omega$
            \begin{align}\label{ilp:bellman_2}
    z_{s, n, o} = \sum_{\substack{s' \in S, \\ n \in N, \\ o' \in \Omega}} 
       \mathbf{T}((s, n, o), (s', n', o')) \cdot l_{s, o, s', n', o'}
\end{align}
        \end{itemize}
    \end{mdframed}
    \end{spacing}
    \medskip
    This MILP is complete, but needs to be improved.
    Note that given a pair of state $s, s' \in S$, the MILP introduces all variables $l_{s, \cdot, s', \cdot, \cdot}$ even if $\mathbf{P}(s, a, s') = 0$ for all $a \in A$.
    We improve the MILP by introducing those variables only when for at least an action $a \in A$, $\mathbf{P}(s, a, s') > 0$.
    , but if for all $a \in A$, it holds that $\mathbf{P}(s, a, s') = 0$, then we do not have to even create the variables $l_{s, \cdot, s', \cdot, \cdot}$ and the constrains involving them. 
    This MILP formulation not only can be used for computing optimal solutions to \OSAtoDGM but also for computing sub-optimal solutions for problem instances for which optimal solutions cannot be computed under limited time budgets. 
    The presented MILP can be directly solved by a variety of highly-optimized MILP solvers. 

\section{Case Studies}
\label{sec:case}
In section, we present results of our implementation of the MILP for several instances of the \OSAtoDGM problem. We implemented our program in Python and used the Python interface of Gorubi~\cite{gurobi} to solve the MILP instances. All experiments were performed on a system with Windows 11, Core i-9 (2000 Mhz)  processor, and 32GB memory.

\subsection{The Reduction Example}
Our implementation for the \OSAtoDGM instance in the reduction, Figure~\ref{fig:reduction}c, compute the sensor alteration in which
\[
 o_4 \rightarrow o_{\clubsuit}, o_2 \rightarrow o_{\clubsuit}, o_1 \rightarrow o_{\clubsuit}.
\]
The MILP for this instance had $59,373$ variables and $236,477$ constraints. It took $0.099$ for the solver to find an optimal solution.
This experiment verifies the correctness of our algorithm.
\subsection{Grid Environment}
This case study considers a robot operating in a $n \times n$ grid environment. An instance where $n=5$, is shown in Figure~\ref{fig:grid_1robot}. The robot is tasked with delivering an item from the starting position $(0, 0)$ to the goal position $(n-1, n-1)$.
The robot has four actions: $N$, $S$, $E$, $W$, which respectively command the robot's  actuators to move to the cell in the North, South, East, and West side of its' current position.
The robot's dynamic is stochastic: The actuators of the robot guarantee that they move the robot to the intended cell specified by the action command with probability $0.8$. 
The probability that the robot's actuators move the robot in either of the two unintended directions orthogonal to its current cell is $0.1$ each, provided movement in both directions is possible. Otherwise, if only one unintended direction is possible (for example, when robot is in one of the corner cells of the environment), the probability of movement in that direction is $0.2$.
%
Figure~\ref{fig:grid_1robot}(bottom) illustrate these two situations.
Generally, the robot is unaware about its own position, i.e., the current state is not observable, but using $7$ range sensors $s_0$ through $s_6$, it has partial observability of the current state of the world.
Each sensor $s_i$ produces an observation $o_i$ when the robot enters a cell guarded by $s_i$. Therefore, there are $7$ observations, each produced by a sensor, along with one additional observation, which we denoted $b$, produced by the cells that are not guarded by any sensors.
With these in mind, the problem is formulated as a POMDP, which has a state for each potential position of the robot---each cell in the grid.
Cell $(2, 2)$, which in the figure is denoted by $X$ in red, is hazardous, and a policy governing the robot movement should help the robot to not enter that cell.
In the POMDP, the states corresponding to the goal and the hazardous position are \emph{absorbing}—states that have self-loops for all actions.

The robot uses the finite-state controller shown in Figure~\ref{fig:grid_1robot}.
The controller tells the robot to do action $E$ when the robot receives any of the observations $o_0$, $o_1$, $o_4$, and $o_6$, and to do action $N$ when any of the observation $o_2$, $o_3$, and $o_4$, are received. The action produced for the blank observation, $b$, will be $E$ if the last non-blank observation was any of the observations $o_0$, $o_1$, $o_4$, and $o_3$, and otherwise, it will be $N$.

Our implementation indicated that the probability of reaching the goal position using this controller without sensor alteration (i.e., a sensor alteration with a cost budget $0$) for the instance with $n=5$ is $0.915$ and the probability of entering the hazardous position is $0.085$.
The purpose of the attacker is to mislead the robot to enter the hazardous state, and hence, the decoy goal contains only that state.
For a cost budget $1$, our implementation of the MILP formulation computed the sensor alteration that changes $o_1$ to $o_3$. This is consistent with the problem instance because altering $o_1$ to $o_3$  maximizes the likelihood of misleading the robot into entering the decoy when only one observation can be altered.
Our program indicated that the probability of reaching the decoy under this sensor alteration is $0.720$, which is significantly higher than $0.085$, the probability of misleading without a sensor attack. To interpret the impact of this sensor alteration, consider that under this sensor alteration, the probability of following the path $(0, 0) \rightarrow (1, 0) \rightarrow (2, 0) \rightarrow (2, 1) \rightarrow (2, 2) \rightarrow (2, 2)$, alone is $0.8^4 = 0.4096$.

\begin{figure}[h]
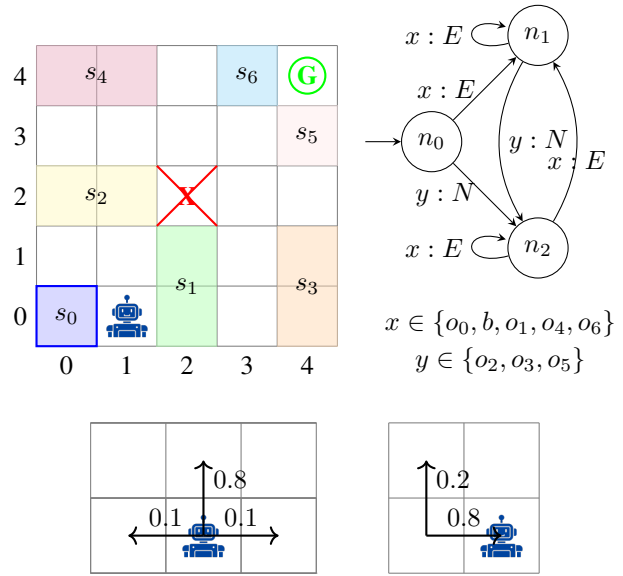

    \centering
    \begin{tikzpicture}[scale=0.8]
                \foreach \x in {0,...,5} {
                    \draw[gray] (\x,0) -- (\x,5);
                    \draw[gray] (0,\x) -- (5,\x);
                }
                %
                \foreach \x in {0,...,4} {
                    \node[below] at (\x+0.5,0) {\x};
                    \node[left] at (0,\x+0.5) {\x};
                }
                %
                \fill[blue!30,opacity=0.5] (0,0) rectangle (1,1); 
                \fill[green!30,opacity=0.5] (2,0) rectangle (3,2); 
                \fill[yellow!30,opacity=0.5] (0,2) rectangle (2,3); 
                \fill[orange!30,opacity=0.5] (4,0) rectangle (5,2); 
                \fill[purple!30,opacity=0.5] (0,4) rectangle (2,5); 
                \fill[pink!30,opacity=0.5] (4,3) rectangle (5,4); 
                \fill[cyan!30,opacity=0.5] (3,4) rectangle (4,5); 
                %
                \draw[red, thick] (2,2) -- (3,3);
                \draw[red, thick] (2,3) -- (3,2);
                %
                \draw[green, thick] (4.5,4.5) circle(0.3);
                %
                \draw[blue, thick] (0,0) rectangle (1,1);
                %
                \node at (0.5,0.5) {$s_0$};
                \node at (2.5,1) {$s_1$};
                \node at (1,2.5) {$s_2$};
                \node at (4.5,1.0) {$s_3$};
                \node at (1,4.5) {$s_4$};
                \node at (4.5,3.5) {$s_5$};
                \node at (3.5,4.5) {$s_6$};
                %
                \node[red] at (2.5,2.5) {\textbf{X}};
                \node[green] at (4.5,4.5) {\textbf{G}};

                \node at (1.5,0.5) {\includegraphics[width=0.6cm]{pics/robot.png}};
            \end{tikzpicture}
            \begin{tikzpicture}[->, >=stealth, initial text=, node distance=2cm, every state/.style={draw, minimum size=0.8cm}]
                \node[state, initial] (n1) {$n_0$};
                \node[state, above right of=n1] (n2) {$n_1$};
                \node[state, below right of=n1] (n3) {$n_2$};
                 \node[below of=n3, xshift=-0.5cm, yshift=1.0cm] (nx) {$x \in \{o_0, b, o_1, o_4, o_6\}$};
                 \node[below of=nx, yshift=1.5cm] (ny) {$y \in \{o_2, o_3, o_5\}$};
                
                \path 
                (n1) edge [left] node {$x: E$} (n2)
                     edge [left] node {$y: N$} (n3);
                
                \path
                (n2) edge[loop left] node {$x: E$} (n2)
                     edge [bend right, right] node {$y: N$} (n3);
                
                \path
                (n3) edge[loop left] node {$x: E$} (n3)
                     edge [bend right] node[below] {$x: E$} (n2);
            \end{tikzpicture}

            \vspace{15pt}
            \begin{tikzpicture}
                \foreach \x in {0,...,3}
                    \foreach \y in {0,...,2}
                {
                    \draw[gray] (\x,0) -- (\x,\y);
                    \draw[gray] (0,\y) -- (3,\y);
                }
                

                \node at (1.5,0.5) {\includegraphics[width=0.6cm]{pics/robot.png}};
                
                \draw[thick, ->] (1.5,0.5) -- (1.5,1.5) node[midway,above right] {$0.8$};
                \draw[thick, ->] (1.5,0.5) -- (2.5,0.5) node[midway,above] {$0.1$};
                \draw[thick, ->] (1.5,0.5) -- (0.5,0.5) node[midway,above] {$0.1$};
            \end{tikzpicture}
            \hspace{20pt}
            \begin{tikzpicture}
                \foreach \x in {0,...,2}
                {
                    \draw[gray] (\x,0) -- (\x,2);
                    \draw[gray] (0,\x) -- (2,\x);
                }
                
                \node at (1.5,0.5) {\includegraphics[width=0.6cm]{pics/robot.png}};
                
                \draw[thick, ->] (0.5,0.5) -- (0.5,1.5) node[midway,above right] {$0.2$};
                \draw[thick, ->] (0.5,0.5) -- (1.5,0.5) node[midway,above] {$0.8$};
                
            \end{tikzpicture}
    \caption{\textbf{Top-left)} A grid environment guarded by $7$ range sensors $s_0$ through $s_7$. The robot is tasked to deliver an item from the starting location $(0, 0)$ to the goal location, $(4, 4)$. Cell $(2, 2)$ is hazardous and must be avoided. That cell considered a decoy and the attacker's purpose is to mislead the robot to that cell. \textbf{Top-right)} A finite-state controller the robot uses. \textbf{Bottom-left)} The robot's dynamic when it performs action N, standing for going to North. \textbf{Bottom-right)} The robot's dynamic when it performs action E, standing for going to East.}
    \label{fig:grid_1robot}
\end{figure}

We repeated the experiment for other cost budgets greater than $1$.
For a cost budget of $2$, our implementation computed to do the alteration $o_1 \rightarrow o_5$ and $o_2 \rightarrow o_0$. This is consistent with the positions of $s_1$ and $s_2$ and the action produced for $o_1$, $o_5$, $o_2$, and $o_0$ by the controller.
The probability of misleading the robot to the decoy under this sensor alteration is $0.861$.
For a cost budget of $3$, our algorithm decided to alter $o_6$ to $o_2$, along with $o_1 \rightarrow o_5$ and $o_2 \rightarrow o_0$. Considering the stochastic nature of the robot's dynamics and controller, altering $o_6$ affects the robot's decision by causing it to take action $N$ instead of $E$ upon observing $o_6$. As a result, the robot may be directed to move through the blank cells to the left and bottom-left of $s_6$ before entering the hazardous state, which could be reached by an infinite number of paths. The probability of following those paths is $\sum_{n=1}^\infty 0.001(0.01)^n \approx 0.00101$. Our implementation indicated that under the computed sensor alteration, the probability of reaching the decoy is $0.862$, rounded to three decimal places.
For a cost budget $4$, the computed sensor alteration included $o_2 \rightarrow o_6$, $o_5 \rightarrow o_1$, $o_1 \rightarrow o_2$, and $o_6 \rightarrow o_3$. Under this alteration, the probability of reaching the decoy was $0.864$ (rounded up).
This probability remained the same for cost budgets greater than $4$. In each case, the actions assigned by the computed sensor alteration for the observations produced by sensors $s_4$ and $s_3$ did not changed. This is reasonable because for example, there is no advantage in choosing action $N$ instead of $E$ upon observing $o_4$.

We performed a scalability experiment by creating $6$ instances of an $n \times n$ grid for $n \in \{5, 15, 25, 35, 45\}$. Each instance maintained the same topology as the case with $n=5$, shown in Figure~\ref{fig:grid_1robot}, the same number of sensors as and the same sensor range sizes.
For each we measured the number variables of the MILP, the number of constraints of the MILP, the time to construct the MILP, and the time to solve the MILP.
Results of this experiment are shown in Figure~\ref{fig:scalability}. 
As expected, both computation and execution times were increased as the size of the grid increased. 
%

\begin{figure}[h!]
\begin{tikzpicture}
    \begin{axis}[
        width=4.8cm,
        height=4cm,
        xlabel={$n$},
        ylabel={Num. of Variables},
        grid=major,
        yticklabel style={anchor=north east},
        y label style={yshift=-15pt}
    ]
    \addplot[red, mark=square*] table[x=n, y=num_vars, col sep=comma] {data.csv};
    \end{axis}
\end{tikzpicture}
\begin{tikzpicture}
    \begin{axis}[
        width=4.8cm,
        height=4cm,
        xlabel={$n$},
        ylabel={Num. of Constraints},
        grid=major,
        yticklabel style={anchor=north east},
        y label style={yshift=-15pt}
    ]
    \addplot[blue, mark=triangle*] table[x=n, y=num_consts, col sep=comma] {data.csv};
    \end{axis}
\end{tikzpicture}
\begin{tikzpicture}
    \begin{axis}[
        width=4.6cm,
        height=4cm,
        xlabel={$n$},
        ylabel={Constr. Time (seconds)},
        grid=major,
        legend pos=north west,
        y label style={yshift=-10pt}
    ]
    \addplot[orange, mark=diamond*] table[x=n, y=const_time, col sep=comma] {data.csv};
    \end{axis}
\end{tikzpicture}
\begin{tikzpicture}
    \begin{axis}[
        width=4.8cm,
        height=4cm,
        xlabel={$n$},
        ylabel={Exec. Time (seconds)},
        grid=major,
        yticklabel style={anchor=north east},
        y label style={yshift=-15pt}
    ]
    \addplot[green, mark=o] table[x=n, y=exec_time, col sep=comma] {data.csv};
    \end{axis}
\end{tikzpicture}
\caption{Results of our scalability experiment for grids similar to the grid in Figure~\ref{fig:grid_1robot}a of size $n \times n$, $n \in \{5, 15, 25, 35, 45\}$.
Note that for the top two graphs, the values on the y-axis are in millions (e.g., the instance for $n = 45$ has more than $1.6$ million variables and more than $6.2$ million constraints).
} 
\label{fig:scalability}
\end{figure}
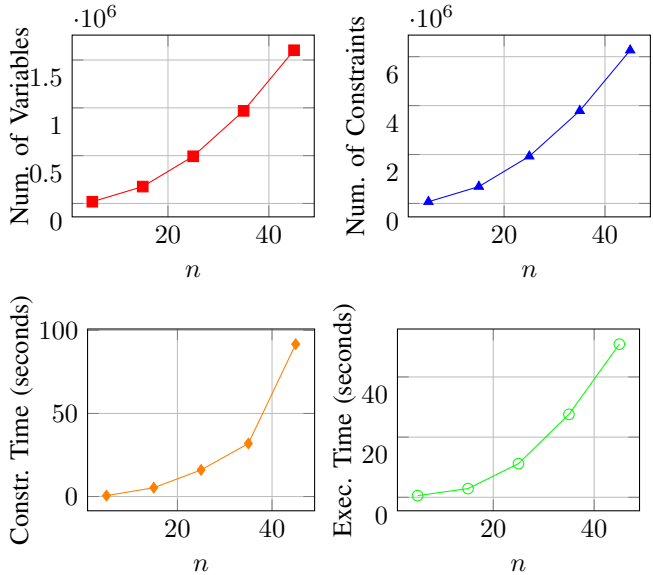
\section{Conclusion}
\label{sec:con}
In this paper, we studied a sensor deception problem the aim of which is to find a sensor alteration that can maximize the probability of misleading an agent to a predetermined decoy goal, under limited capability for sensor alteration. 
The environment is modeled by a POMDP and the agent's actions are governed by a Finite State Controller (FSC). 
We proved that our problem is $\NP$-hard, and provided an algorithm based on MILP. We showed through experiments that our algorithm is capable of computing optimal solutions for problems of moderate size.
%
Future work can focus on improving the MILP or introducing a new one with less variables and constraints. The standard techniques to deal with $\NP$-hard problems, such as introducing hueristic and approximate solutions, as wells as identifying instances that can be solved in polynomial time, can be considered for future work. Also, it might be useful to study the problem for more general kinds of models and strategies. 
\bibliography{main}


\end{document}